\documentclass[letterpaper,11pt]{article}

\usepackage{fullpage}

\usepackage{graphicx}
\usepackage{subcaption}

\usepackage[square,comma,numbers,sort&compress]{natbib}

\usepackage[breaklinks,colorlinks,
            linkcolor=darkblue,citecolor=darkgreen,urlcolor=darkgreen]
            {hyperref}
\usepackage{color}
\definecolor{darkdarkgreen}{rgb}{0,0.35,0}
\definecolor{darkgreen}{rgb}{0,0.4,0}
\definecolor{darkblue}{rgb}{0,0,0.6}
\definecolor{lightgray}{rgb}{0.95,0.95,0.95}
\usepackage{hypernat}

\newcommand{\namecite}[1]{\citet{#1}}
\newcommand{\citeeg}[1]{\cite[e.g.,][]{#1}}
\newcommand{\citecf}[1]{\cite[cf.][]{#1}}

\usepackage{amssymb}
\usepackage{amsmath}
\usepackage{amsthm}
\usepackage{epstopdf}
\usepackage{color}
\usepackage{url}

\usepackage[vlined,boxed]{algorithm2e}

\newtheorem{theorem}{Theorem}
\newtheorem{proposition}{Proposition}
\newtheorem{lemma}{Lemma}

\theoremstyle{definition}
\newtheorem{definition}{Definition}

\newtheorem{remark}{Remark}

\newcommand\argmax{\mathrm{arg} \max}
\newcommand\bigOh{\mathrm{O}}
\newcommand\ltlOh{\mathrm{o}}
\newcommand\cost{\mathit{cost}}

\newcommand\ub{\mathit{ub}}

\title{Selecting Near-Optimal Learners via Incremental Data Allocation}

\author{
Ashish Sabharwal\\
Allen Institute for AI\\
Seattle, WA, USA\\
\texttt{\small AshishS@allenai.org}
\and
Horst Samulowitz,\ Gerald Tesauro\\
IBM Watson Research Center\\
Yorktown Heights, NY, USA\\
\texttt{\small samulowitz,gtesauro@us.ibm.com}
}

\begin{document}

\date{}

\maketitle


\begin{abstract}
We study a novel machine learning (ML) problem setting of sequentially allocating small subsets of training data amongst a large set of classifiers. The goal is to select a classifier that will give near-optimal accuracy when trained on all data, while also minimizing the cost of misallocated samples. This is motivated by large modern datasets and ML toolkits with many combinations of learning algorithms and hyper-parameters. Inspired by the principle of ``optimism under uncertainty,'' we propose an innovative strategy, Data Allocation using Upper Bounds (DAUB), which robustly achieves these objectives across a variety of real-world datasets.

We further develop substantial theoretical support for DAUB in an idealized setting where the expected accuracy of a classifier trained on $n$ samples can be known exactly. Under these conditions we establish a rigorous sub-linear bound on the regret of the approach (in terms of misallocated data), as well as a rigorous bound on suboptimality of the selected classifier. Our accuracy estimates using real-world datasets only entail mild violations of the theoretical scenario, suggesting that the practical behavior of DAUB is likely to approach the idealized behavior. 

\end{abstract}


\section{Introduction}

The goal of our work is to develop novel practical methods
to enhance tractability of Data Science practice in the era of
Big Data.  Consider, for example, the following very common scenario:
A Data Science practitioner is given a data set comprising a training set,
a validation set, and a collection of classifiers in an ML toolkit,
each of which may have numerous possible hyper-parameterizations.
The practitioner would like to determine which classifier/parameter combination
(hereafter referred to as ``learner'') would yield the highest validation
accuracy, after training on all examples in the training set.  However,
the practitioner may have quite limited domain knowledge of salient
characteristics of the data, or indeed of many of the algorithms
in the toolkit.

In such a scenario, the practitioner may inevitably resort to the
traditional approach to finding the best
learner~\citecf{caruana2006empirical}, namely,
brute-force training of all learners on the full training set,
and selecting the one with best validation accuracy.
Such an approach is acceptable if the computational cost of training
all learners is not an issue.
However, in the era of Big Data, this is becoming increasingly infeasible.
Web-scale datasets are proliferating from sources such as Twitter, TREC,
SNAP, ImageNet, and the UCI repository, particularly in domains such as
vision and NLP.  ImageNet datasets can exceed 100 gigabytes, and the
recent ``YouTube-Sports-1M'' video collection exceeds 40 terabytes.
Moreover, the diverse set of learners available in today's
ML packages~\cite{weka,scikit-learn,pybrain,mallet}
are continually expanding, and many of the most successful recent
algorithms entail very heavy training costs (e.g., Deep Learning neural
nets with Dropout).

The above factors motivate a search for techniques to reduce training
cost while still reliably finding a near-optimal learner.
One could consider training each learner on
a small subset of the training
examples, and choose the best performing one.
This entails less computation, but could result
in significant loss of learner accuracy,
since performance on a small subset can be a
misleading predictor of performance on the full dataset.
As an alternative, the small-subset results could be projected forward
using parameterized accuracy models to predict full training set accuracy.
Creating such models is, however, a daunting task~\cite{PredictingModel2008}, 
potentially needing prior knowledge about learners and domain,
characteristic features of the data, etc.

In this paper, we develop a novel formulation of what it means to solve
the above dual-objective problem, and we present a novel solution approach,
inspired by multi-armed bandit literature
\cite{acf02:ucb1J,thompson33:sampling,scott2010bayesbandits,agrawal2011analysis}.
Our method develops model-free, cost-sensitive strategies for
sequentially allocating small batches of training data to selected
learners, wherein ``cost'' reflects misallocated
samples that were used to train other learners that were ultimately not
selected.  We express the cost in terms of the \emph{regret} of the approach,
comparing the algorithm's cost with that of an oracle which only
allocates data to the best learner.

Our main contributions are as follows.
First, we give a precise definition of a new ML problem setting,
called the \emph{Cost-Sensitive Training Data Allocation
Problem}.  Second, we present a simple, knowledge-free, easy-to-use
and practical new algorithm for this setting, called DAUB (Data
Allocation with Upper Bounds).
Third, we give empirical demonstrations that DAUB achieves significant
savings in training time while reliably achieving optimal or
near-optimal learner accuracy over multiple real-world datasets.
Fourth, we provide theoretical support for DAUB in an idealization
of the real-world setting, wherein DAUB can work with noiseless accuracy
estimates when training on $n$ samples, in lieu of actual noisy estimates.
The real-world behavior of DAUB will progressively approach the idealized
behavior as $n$ becomes large.  In this setting, we establish a bound on
accuracy of learners selected by DAUB, a sub-linear bound on the data
misallocated by DAUB, and an associated bound on the
computational training cost (regret).





Related work on traditional bandit strategies mentioned above,
such as the celebrated UCB1~\cite{acf02:ucb1J}
and Thompson sampling~\cite{thompson33:sampling,agrawal2011analysis},
presume that additional trials of a given arm yield
stationary payoffs. Whereas in our scenario, additional data allocations
to a learner yield increasing values of its accuracy.
There are also existing
methods to optimize a single arbitrary function while minimizing the
number of evaluations~\citecf{munos2013:optimism}. These also do not
fit our setting: we are dealing with multiple unknown but
well-behaved functions, and wish to rank them on estimated accuracy
after training on the full dataset, based on
their upper-bounds from much fewer samples.

Somewhat related is
algorithm portfolio selection~\cite{Rice} which
seeks the most suitable algorithm (e.g.,
learner) for a given problem instance, based on knowledge from
other instances and features characterizing the
current instance. Note, however, that most selection algorithms use
parameterized accuracy models which are fit
to data \citeeg{algorithmruntimeprediction}.
Also related is work on hyper-parameter optimization,
where one searches for novel
configurations of algorithms to improve performance
\cite{PAOML,MSMS,randomsearch,algsforhyperoptimization}
or a combination of both~\cite{HyperOptviaMetaLearning}.
An example is
Auto-Weka~\cite{AutoWEKA}, which combines selection and parameter
configuration based on Bayesian optimization
\citecf{bayesianoptimization}. Predicting generalization error
on unseen data has in fact been recognized as a
major ML challenge~\cite{Guyon2006:challenge}.

A recent non-frequentist
approach~\cite{BayesianBandits} takes a Bayesian view of
multi-armed bandits, applicable especially when the number of
arms exceeds the number of allowed evaluations, and applies it also to
automatic selection of ML algorithms. Like some
prior methods, it evaluates algorithms on a small
fixed percentage (e.g., 10\%) of the full dataset.
Unlike the above approaches, 
we do not assume that training (and evaluation) on a small fixed fraction
of the data reliably ranks full-training results.

Finally, \namecite{domhanspeeding} recently proposed extrapolating
learning curves to enable early termination of non-promising learners.
Their method is designed specifically for neural networks and does not
apply directly to many classifiers (SVMs, trees, etc.) that train
non-iteratively from a single pass through the dataset.  They also do
not focus on a theoretical justification and fit accuracy estimates to
a library of hand-designed learning curves.

\section{Cost-Sensitive Training Data Allocation}

We begin by formally defining the problem of cost-sensitive training data allocation.
As before, we use \emph{learner} to refer to a classifier along with a
hyper-parameter setting for it.
Let $\mathcal{C} = C_1, C_2, \ldots, C_M$ be a set of $M$
learners which can be trained on subsets of a training set $T_r$ and
evaluated on a validation set $T_v$. Let $|T_r| = N$.
For $k \in \mathbb{N}$, let $[k]$
denote the set $\{1, 2, \ldots, k\}$.

For $i \in [M]$, let $c_i : [N]
\to \mathbb{R}^+$ be a \emph{cost function} denoting expected computational
cost of training learner $C_i$ when $n$ training
examples are drawn uniformly at random from $T_r$.\footnote{While we
  define the core concepts in terms of expected values suitable for a
  formal definition and idealized analysis, the actual DAUB algorithm
  will operate on observed values of $c_i$ on particular subsets of
  $n$ training examples chosen at runtime.}
We make two common assumptions about the training process, namely,
that it looks at all training data and its complexity grows at least
linearly. Formally, $c_i(n) \geq n$ and $c_i(m) \geq \frac{m}{n}
c_i(n)$ for $m > n$.

For $i \in [M]$, let $f_i : [N] \to [0,1]$ be an \emph{accuracy
  function} where $f_i(n)$ denotes expected accuracy of $C_i$ on $T_v$
when trained on $n$ training examples chosen at random from $T_r$.
The corresponding \emph{error function}, $e_i(n)$, is defined as $1 -
f_i(n)$. Note that our tool also supports accuracy functions not tied
to a fixed validation set $T_v$ (e.g., cross-validation) and other
measures such as precision, recall, and F1-score; our analysis applies
equally well to these measures.

We denote a training data \emph{allocation} of $n$ training samples
to learner $i$ by a pair $a = (i, n)$.
Let $S = \left( (i^{(1)},n^{(1)}), (i^{(2)}, n^{(2)}), \ldots,
  (i^{(s)}, n^{(s)}) \right)$ be a \emph{sequence of allocations} to
learners in $\mathcal{C}$. We will use $S_i$ to denote the
\emph{induced subsequence} containing all training data allocations to
learner $C_i$, i.e., the subsequence of $S$ induced by all pairs
$(i^{(k)},n^{(k)})$ such that $i^{(k)} = i$. In our context,
if allocations $(i,n^{(k)})$ and
$(i,n^{(\ell)})$ are in $S_i$ with $k < \ell$, then $n^{(k)} <
n^{(\ell)}$.

Evaluating $f_i(n)$ amounts to training learner $C_i$ on $n$
examples from $T_r$ and evaluating its accuracy. This, in expectation, incurs a
computational cost of $c_i(n)$. In general, the expected \emph{training
  complexity} or cost associated with $\mathcal{C}$ under the data
allocation sequence $S$ is $\cost(S) = \sum_{(i,n) \in S} c_i(n)$.

Our goal is to search for an $\tilde{i} \in [M]$ such that
$f_{\tilde{i}}(N)$ is maximized, while also ensuring that overall
training cost
is not too large relative to
$c_{\tilde{i}}(N)$. This bi-objective criterion is not easy to
achieve. E.g., a brute-force evaluation, corresponding to
$S = \left( (1,N), (2,N), \ldots, (M,N)
\right)$ and $\tilde{i} = \argmax_{i \in [M]} f_i(N)$, obtains the
optimal $\tilde{i}$ but incurs maximum training cost of $c_i(N)$
for all suboptimal learners. On the other hand, a low-cost
heuristic $S = \left( (1,n), (2,n), \ldots, (m,n),
  (\tilde{i},N) \right)$ for some $n \ll N$ and $\tilde{i} = \argmax_i
f_i(n)$, incurs a small training overhead of only $c_i(n)$ for each
suboptimal $C_i$, but may choose an arbitrarily suboptimal $\tilde{i}$.

We seek an in-between solution,
ideally with the best of both worlds: a bounded optimality gap on
$C_{\tilde{i}}$'s accuracy, and a bounded regret in
terms of data misallocated to sufficiently suboptimal
learners. Informally speaking, we will ensure that learners with
performance at least $\Delta$ worse than optimal are allocated only
$\ltlOh(N)$ training examples, i.e., an asymptotically vanishing
fraction of $T_r$. Under certain conditions, this will ensure that the
training cost regret is sublinear. We next formally define the notions
of suboptimality and regret in this context.

\begin{definition}
  \label{def:optimality}
  Let $\mathcal{C}$ be a collection of $M$ learners with accuracy
  functions $f_i$, $\Delta \in (0,1]$, $n \in \mathbb{N}^+$, and $i^* =
  \argmax_{i \in [M]} f_i(n)$. A learner $C_j \in \mathcal{C}$ is
  called \emph{$(n,\Delta)$-suboptimal} for $\mathcal{C}$ if
  $f_{i^*}(n) - f_j(n) \geq \Delta$, and \emph{$(n,\Delta)$-optimal}
  otherwise.
\end{definition}

\begin{definition}
  \label{def:regret}
  Let $S$ be a data allocation sequence for a collection
  $\mathcal{C}$ of learners with accuracy functions $f_i$, $\Delta
  \in (0,1]$, and $n \in \mathbb{N}^+$. The
  \emph{$(n,\Delta)$-regret} of $S$ for $\mathcal{C}$ is defined as:
  \[
      \sum_{i : C_i \text{\ is\ } (n,\Delta)\text{-suboptimal}} \cost(S_i)
  \]
\end{definition}

The regret of $S$ is thus the cumulative cost of
training all $(n,\Delta)$-suboptimal learners when using $S$.

\begin{definition}[COST-SENSITIVE TRAINING DATA ALLOCATION PROBLEM]
  \label{def:problem}
  Let $\mathcal{C} = \{C_1, \ldots, C_M\}$ be a set of learners,
  $T_r$ be a training set for $\mathcal{C}$ containing $N$ examples,
  $T_v$ be a validation set, $c_i$ and $f_i$, for $i \in [M]$, be
  the training cost and accuracy functions, resp., for learner
  $C_i$, and $\Delta \in (0,1]$ be a constant.
  The \emph{Cost-Sensitive Training Data Allocation Problem} is to compute a training
  data allocation sequence $S$ for $\mathcal{C}$ and $T_r$ as well as
  a value $\tilde{i} \in [M]$ such that:
  \begin{enumerate}
  \item $S$ contains $(\tilde{i},N)$,
  \item $C_{\tilde{i}}$ is $(N,\Delta)$-optimal,
  \item $\cost(S_{\tilde{i}}) \leq d \cdot c_{\tilde{i}}(N)$ for some fixed
    constant $d$, and
  \item $(N,\Delta)$-regret of $S$ is $\ltlOh(M \cdot \cost(S_{\tilde{i}}))$
    in $N$.
  \end{enumerate}
\end{definition}

A solution to this problem thus identifies an
$(N,\Delta)$-optimal learner $C_{\tilde{i}}$, trained
on all of $T_r$, incurring on $C_{\tilde{i}}$ no more than a constant
factor overhead relative to the minimum training cost of
$c_{\tilde{i}}(N)$, and with a guarantee that any
$(N,\Delta)$-suboptimal learner $C_i$ incurred a vanishingly small
training cost compared to training $C_{\tilde{i}}$ (specifically,
$\cost(S_i) / \cost(S_{\tilde{i}}) \to 0$ as $N \to \infty$).

\section{The DAUB Algorithm}
\label{sec:optimistic}

Algorithm~\ref{alg:daub} describes our \emph{Data Allocation using
Upper Bounds} strategy.  The basic idea is to project an optimistic
upper bound on full-training accuracy $f_i(N)$ of learner $i$ using
recent evaluations $f_i(n)$.  The learner with highest upper bound
is then selected to receive additional samples.
Our implementation of DAUB uses monotone regression to estimate
upper bounds on $f_i(N)$ as detailed below,
since observed accuracies are noisy and may occasionally violate known
monotonicity of learning curves. Whereas in the noise-free setting,
a straight line through the two most recent values of $f_i(n)$ 
provides a strict upper bound on $f_i(N)$.

\newcommand{\pluseq}{\mathrel{+}=}
\newcommand{\minuseq}{\mathrel{-}=}

\begin{algorithm}[t]
\begin{small}
\dontprintsemicolon
\SetKwInOut{Input}{\it Input}
\SetKwInOut{Params}{\it Params}
\SetKwInOut{Output}{\it Output}
\SetKwFunction{DAUB}{\textbf{DAUB}}
\SetKwFunction{TrainClassifier}{TrainLearner}
\SetKwFunction{UpdateBound}{UpdateBound}
\SetKwFunction{EstimateSlope}{EstimateSlope}
\Input{Learners $\mathcal{C} = \{C_1, \ldots, C_M\}$, training
  examples $T_r$, $N = |T_r|$, validation set $T_v$}
\Output{Learner $C_{\tilde{i}}$ trained on $T_r$, data allocation sequence $S$}
\Params{Geometric ratio $r > 1$, granularity $b \in \mathbb{N}^+$
  s.t.\ $b r^2 \leq N$}
\BlankLine\;
\DAUB{$\mathcal{C}, T_r, T_v, r, b$}\;
\Begin{
  $S \leftarrow$ empty sequence\;
  \For {$i \in 1..M$} {
    \For {$k \in 0..2$} {
      append $(i,b r^k)$ to $S$\;
      \TrainClassifier($i, b r^k$)\;
    }
    $n_i \leftarrow b r^2$\;
    $u_i \leftarrow$ \UpdateBound($i, n_i$)\;
  }
  \While {$(\max_i n_i) < N$} {
    $j \leftarrow \argmax_{i \in [M]} u_i$\ \ \ \ (break ties arbitrarily)\;
    $n_j \leftarrow \min \{\lceil r n_j \rceil, N\}$\;
    append $(j,n_j)$ to $S$\;
    \TrainClassifier($j, n_j$)\;
    $u_j \leftarrow$ \UpdateBound($j$)\;
  }
  select $\tilde{i}$ such that $n_{\tilde{i}} = N$\;
  \Return $C_{\tilde{i}}, S$\;
}
\BlankLine\;
\textbf{sub} \TrainClassifier{$i \in [M], n \in [N]$}\;
\Begin{
  $T \leftarrow n$ examples sampled from $T_r$\;
  Train $C_i$ on $T$\;
  $f^T_i[n] \leftarrow$ training accuracy of $C_i$ on $T$\;
  $f^V_i[n] \leftarrow$ validation accuracy of $C_i$ on $T_v$\;
  \If {$n/r \geq b$} {
    $\delta \leftarrow (f^V_i[n] - f^V_i[n/r])$\;
    \If {$\delta < 0$} {
      $f^V_i[n/r] \minuseq \delta / 2$\;
      $f^V_i[n] \pluseq \delta / 2$
    }
  }
}
\BlankLine\;
\textbf{sub} \UpdateBound{$i \in [M], n \in [N]$}\;
\Begin{
  $f'^V_i[n] \leftarrow$ LinearRegrSlope($f^V_i[n/r^2], f^V_i[n/r], f^V_i[n]$)\;
  $\ub^V_i[n] \leftarrow f^V_i[n] + (N-n) f'^V_i[n]$\;
  \Return $\min \{f^T_i[n], \ub^V_i[n]\}$\;
}
\end{small}
\caption{\label{alg:daub} Data Allocation using Upper Bounds}
\end{algorithm}



As a bootstrapping step, DAUB first allocates
$b, br,$ and $br^2 \leq N$ training examples to each learner $C_i$,
trains them, and records their training and validation accuracy in
arrays $f^T_i$ and $f^V_i$, resp.
If $f^V_i$ at the current point is
smaller than at the previous point,
DAUB uses a simple \emph{monotone regression} method,
making the two values meet in the middle.

After bootstrapping, in each iteration, it identifies a learner
$C_j$ that has the most promising upper bound estimate (computed as
discussed next) on the unknown projected expected accuracy $f_j(N)$
and allocates $r$ times more examples (up to $N$) to it than what
$C_j$ was allocated previously. For computing the upper bound
estimate, DAUB uses two sources. First, assuming training and
validation data come from the same distribution, $f^T_i[n_i]$ provides
such an estimate. Further, as will be justified in the analysis of the
idealized scenario called DAUB*, $\ub^V_i[n_i] = f^V_i[n_i] + (N-n_i)
f'^V_i[n_i]$ also provides such an estimate under certain conditions,
where $f'^V_i[n_i]$ is the estimated derivative computed as the slope
of the linear regression best fit line through $f^V_i[n]$ for $n \in
\{n_i/r^2, n_i/r, n_i\}$. Once some learner $C_{\tilde{i}}$ is
allocated all $N$ training examples, DAUB halts and outputs
$C_{\tilde{i}}$ along with the allocation sequence it used.



\subsection{Theoretical Support for DAUB}

To help understand the behavior of DAUB,
we consider an idealized variant, DAUB*, that operates precisely
like DAUB but \emph{has access to the true expected accuracy and cost
functions, $f_i(n)$ and $c_i(n)$, not just their observed estimates.}
As $n$ grows, learning variance (across random
batches of size $n$) decreases, observed estimates of $f_i(n)$
and $c_i(n)$ converge to these ideal values, and the behavior of DAUB
thus approaches that of DAUB*.

Let $f^* = \max_{i \in [M]} f_i(N)$ be the (unknown) target accuracy
and $C_{i*}$ be the corresponding (unknown) optimal learner.  For
each $C_i$, let $u_i : [N] \to [0,1]$ be an arbitrary projected upper
bound estimate that DAUB* uses for $f_i(N)$ when it has allocated $n <
N$ training examples to $C_i$.  We will assume w.l.o.g.\ that $u_i$ is
non-increasing at the points where it is evaluated by
DAUB*.\footnote{Since DAUB* evaluates $u_i$ for increasing
  values of $n$, it is easy to enforce monotonicity.}  For the initial
part of the analysis, we will think of $u_i$ as a black-box function,
ignoring how it is computed.  Let $u^{\text{min}}(N) = \min_{i \in
  [N]} \{u_i(N)\}$.  It may be verified that once $u_j(n)$ drops
below $u^{\text{min}}(N)$, DAUB* will stop allocating more samples to
$C_j$.
%
%
While this gives insight into the behavior of DAUB*, for the analysis
we will use a slightly weaker form $n^*_i$ that depends on the target accuracy
$f^*$ rather than $u^{\text{min}}(N)$.

\begin{definition}
\label{def:valid-proj-ub}
$u_i : [N] \to [0,1]$ is a \emph{valid projected upper bound function} if $u_i(n) \geq
f_i(N)$ for all $n \in [N]$.
\end{definition}

\begin{definition}
\label{def:n-star}
Define $n^*_i \in \mathbb{N}$ as $N$ if $u_i(N) \geq f^*$ and as $\min
\{\ell \mid u_i(\ell) < f^*\}$ otherwise.
\end{definition}

A key observation is that when using $u_j$ as the only source of
information about $C_j$, \emph{one must allocate at least $n^*_j$
  examples to $C_j$} before acquiring enough information to conclude
that $C_j$ is suboptimal. Note that $n^*_j$ depends on the interaction
between $u_j$ and $f^*$, and is thus unknown. Interestingly, we can
show that DAUB*$(\mathcal{C}, T_r, T_v, r, b)$ allocates to $C_j$ at
most a constant factor more examples, specifically fewer than $r
n^*_j$ in each step and $\frac{r^2}{r-1} n^*_j$ in total, if it has
access to valid projected upper bound functions for $C_j$ and
$C_{i^*}$ (cf.~Lemma~\ref{lem:data-allocation-new} in Appendix).  In
other words, \emph{DAUB*'s allocation is essentially optimal} w.r.t.\
$u_j$.

\begin{remark}
A careful selection of the learner in each round is
critical for allocation optimality w.r.t.\ $u_j$. Consider a simpler
alternative: In round $k$, train all currently active learners on $n = b r^k$
examples, compute all $f_i(n)$ and $u_i(n)$, and permanently drop
$C_j$ from consideration if $u_j(n) < f_i(n)$ for some $C_i$. This
will not guarantee allocation optimality; any permanent decisions
to drop a classifier must necessarily be conservative to be correct.
By instead only temporarily suspending suboptimal looking
learners, DAUB* guarantees a
much stronger property: $C_j$ receives no more allocation as soon as
$u_j(n)$ drops below the (unknown) target $f^*$.
\end{remark}

The following 
observation connects data allocation to
training cost: if DAUB* allocates at most $n$ training examples to a
learner $C_j$ in each step, then its overall cost for $C_j$ is at most
$\frac{r}{r-1} c_j(n)$ (cf.~Lemma~\ref{lem:data-vs-cost-new} in
Appendix). Combining this with Lemma~\ref{lem:data-allocation-new}, we
immediately obtain the following result regarding DAUB*'s
regret:\footnote{All proofs are deferred to the Appendix.}

\begin{theorem}
  \label{thm:regret-new}
  Let $\mathcal{C}, T_r, T_v, N, M, c_i$ and $f_i$ for $i \in [M]$ be
  as in Definition~\ref{def:problem}.  Let $r > 1, b \in
  \mathbb{N}^+,$ and $S$ be the 
  allocation sequence
  produced by DAUB*$(\mathcal{C}, T_r, T_v, r, b)$.  If the projected
  upper bound functions $u_j$ and $u_{i*}$ used by DAUB* are valid,
  then $\cost(S_j) \leq \frac{r}{r-1} c_j(r n^*_j)$.
\end{theorem}

In the remainder of the analysis, we will (a) study the validity of
the actual projected upper bound functions used by DAUB* and (b)
explore conditions under which $(N,\Delta)$-suboptimality of $C_j$
guarantees that $n^*_j$ is a vanishingly small fraction of $N$,
implying that DAUB* incurs a vanishingly small training cost on any
$(N,\Delta)$-suboptimal learner.

\subsubsection{Obtaining Valid Projected Upper Bounds}

If $f_i$ for $i \in [M]$ were arbitrary functions, it would clearly be
impossible to upper bound $f_i(N)$ by looking only at estimates of
$f_i(n)$ for $n < N$. Fortunately, each $f_i$ is the expected accuracy
of a learner and is thus expected to behave in a certain way. In
order to bound DAUB*'s regret, we make two assumptions on the behavior
of $f_i$. First, $f_i$ is non-decreasing, i.e., more training data
does not hurt validation accuracy. Second, $f_i$ has a
\emph{diminishing returns} property, namely, as $n$ grows, the
additional validation accuracy benefit of including more training
examples diminishes. Formally:

\begin{definition}
  \label{def:wellbehaved}
  $f: \mathbb{N} \to [0,1]$ is
  \emph{well-behaved} if it is non-decreasing and its discrete
  derivative, $f'$, is non-increasing.
\end{definition}

These assumptions on expected accuracy
are well-supported from
the PAC theory perspective.  Let $\ub_i(n)$ be the projected upper
bound function used by DAUB* for $C_i$,
namely the minimum of the \emph{training accuracy} $f^T_i(n)$ of $C_i$
at $n$ and the \emph{validation accuracy} based expression $f_i(n) +
(N-n) f'_i(n)$. For DAUB*, we treat $f'_i(n)$ as the one-sided
discrete derivative defined as $(f_i(n) - f_i(n-s))/s$ for some
parameter $s \in \mathbb{N}^+$.
 We assume the training and validation sets, $T_r$ and
$T_v$, come from the same distribution, which means $f^T_i(n)$ itself
is a valid projected upper bound. Further, we can show that if
$f_i(n)$ is well-behaved, then \emph{$\ub_i(n)$ is a valid projected
  upper bound function} (cf.~Lemma~\ref{lem:ub} in Appendix).


Thus, instead of relying on a parameterized functional form to model
$f_i(n)$, DAUB* evaluates $f_i(n)$ for certain values of $n$ and
computes an expression that is guaranteed to be a valid upper bound on
$f_i(N)$ if $f_i$ is well-behaved.

\subsubsection{Bounding Regret}

We now fix $\ub_i$ as the projected upper bound functions and explore
how $(N,\Delta)$-suboptimality and the well-behaved nature of $f_i$
together limit how large $n^*_i$ is.

\begin{definition}
  For $\Delta \in (0,1]$ and a well-behaved accuracy function $f_i$,
  define $n^\Delta_i \in \mathbb{N}$ as $N$ if $f'_i(N) > \Delta/N$
  and as $\min \{\ell \mid f'_i(\ell) \leq \Delta/N\}$ otherwise.
\end{definition}

Using first order Taylor expansion, we can prove that
$\ub_i(n^\Delta_i) < f^*$, implying $n^*_j \leq n^\Delta_j$ (cf.\
Lemma~\ref{lem:n-delta} in Appendix).  Combining this with
Theorem~\ref{thm:regret-new}, we obtain:

\begin{theorem}
  \label{thm:regret-actual}
  Let $\mathcal{C}, T_r, T_v, N, M, c_i$ and $f_i$ for $i \in [M]$ be
  as in Definition~\ref{def:problem}.  Let $r > 1, b \in \mathbb{N}^+,
  \Delta \in (0,1], C_j \in \mathcal{C}$ be an $(N,\Delta)$-suboptimal
  learner, and $S$ be the
  allocation sequence
  produced by DAUB*$(\mathcal{C}, T_r, T_v, r, b)$.  If $f_j$ and
  $f_{i^*}$ are well-behaved, then $\cost(S_j) \leq \frac{r}{r-1}
  c_j(r n^\Delta_j)$.
\end{theorem}

The final piece of the analysis is an asymptotic bound on
$n^\Delta_i$.  To this end, we observe that the derivative of any
bounded, well-behaved, discrete function of $\mathbb{N}$ behaves
asymptotically as $\ltlOh(1/n)$ (cf.~Proposition~\ref{prop:fast-decay}
in Appendix). Applying this to $f_j$, we can prove
that if $f'_j(N) \leq \Delta/N$, then $n^\Delta_j$ is $\ltlOh(N)$ in
$N$ (cf.\ Lemma~\ref{lem:n-delta-asymptotic} in Appendix).

This leads to our main result regarding DAUB*'s regret:

\begin{theorem}[Sub-Linear Regret]
  \label{thm:regret-bound-new}
  Let $\mathcal{C}, T_r, T_v, N, M, c_i$ and $f_i$ for $i \in [M]$ be
  as in Definition~\ref{def:problem}. Let $r > 1, b \in \mathbb{N}^+,$
  and $\Delta \in (0,1]$. Let $J = \{j \mid C_j \in \mathcal{C}
  \text{\ is\ } (N,\Delta)\text{-suboptimal}\}$. For all $j \in J$,
  suppose $f_j$ is well-behaved and
  $f'_j(N) \leq \Delta/N$. If DAUB*$(\mathcal{C}, T_r, T_v, r, b)$
  outputs $S$ as the training data allocation sequence along with a
  selected learner $C_{\tilde{i}}$ trained on all of $T_r$, then:
  \begin{enumerate}
    \item $\cost(S_{\tilde{i}}) \leq \frac{r}{r-1} c_{\tilde{i}}(N)$;

    \item $(N,\Delta)$-regret of $S$ is $\ltlOh(\sum_{j \in J}
      c_j(N))$ in $N$; and

    \item If $c_j(n) = \bigOh(c_{\tilde{i}}(n))$ for all $j \in J$,
      then the $(N,\Delta)$-regret of $S$ is
      $\ltlOh(M \cdot \cost(S_{\tilde{i}}))$ in $N$.
  \end{enumerate}
\end{theorem}

Thus, DAUB* successfully solves the cost-sensitive training data
allocation problem whenever for $j \in J$,
$f_j$ is well-behaved
and $c_j(n) = \bigOh(c_{\tilde{i}}(n))$,
i.e., training any suboptimal learner is asymptotically not any costlier
than training an optimal learner. While more refined versions of
this result can be generated, the necessity of an assumption on the
cost function is clear: if a suboptimal learner $C_j$ was arbitrarily
costlier to train than optimal learners, then, in order to
guarantee near-optimality, one must incur a significant misallocation
cost training $C_j$ on some reasonable subset of $T_r$ in order to
ascertain that $C_j$ is in fact suboptimal.


\subsubsection{Tightness of Bounds}

The cost bound on misallocated data in Theorem~\ref{thm:regret-actual}
in terms of $n^\Delta_i$ is in fact tight (up to a constant factor) in
the worst case, unless further assumptions are made about the accuracy
functions. In particular, every algorithm that guarantees
$(N,\Delta)$-optimality without further assumptions must, in the worst
case, incur a cost of the order of $c_j(n^\Delta_j)$ for every
suboptimal $C_j \in \mathcal{C}$
(cf.~Theorem~\ref{thm:data-allocation-LB} in Appendix for a formal
statement):

\begin{theorem}[Lower Bound, informal statement]
  \label{thm:data-allocation-LB-informal}
  Let $\Delta \in (0,1]$ and $\mathcal{A}$ be a training data
  allocation algorithm that always outputs an $(N,\Delta)$-optimal
  learner. Then there exists an $(N,\Delta)$-suboptimal learner $C_j$
  that would force $\mathcal{A}$ to incur a misallocated training cost
  larger than $ c_j(n^\Delta_j) / 2$.
\end{theorem}

\section{Experiments}

Our experiments make use of 41 classifiers covering a wide range of
algorithms (SVMs, Decision Trees, Neural Networks, Logistic Regression, etc.)
as implemented in WEKA~\cite{weka}.
All experiments were conducted on AMD Opteron 6134
machines with 32 cores and 64 GB memory, running Scientific Linux
6.1.\footnote{Code and data, including full parameterization
  for each classifier, are available from the authors.}


We first evaluate DAUB on one real-world binary classification dataset,
``Higgs boson''~\cite{HiggsDataset2014} and one artificial dataset,
``Parity with distractors,'' to examine robustness of DAUB's strategy
across two extremely different types of data.  In the latter task the
class label is the parity of a (hidden) subset of binary features---the
remaining features serve as distractors, with no influence on the class
label. We generated 65,535 distinct examples based on 5-bit parity with
11 distractors, and randomly selected 21,500 samples each for $T_r$ and $T_v$.
For the Higgs and other real-world datasets, we first randomly split
the data with a 70/30 ratio and selected 38,500
samples for $T_r$ from the 70\% split and use the 30\% as $T_v$ .
We coarsely optimized the DAUB parameters at $b=500$ and $r=1.5$ based on
the Higgs data, and kept those values fixed for all datasets.  This yielded
11 possible allocation sizes:
500, 1000, 1500, 2500, 4000, 5000, 7500, 11500, 17500, 25500,
38500\footnote{Unfortunately some of our classifiers crash and/or
run out of memory above 38500 samples.}.

\begin{table}[htb]
 \caption{\label{tab:performance} Comparison of full training, DAUB
    without training accuracy bound $f^T$, and DAUB.}  
  \setlength{\doublerulesep}{\arrayrulewidth}
  \setlength{\tabcolsep}{1ex}
  \centering
       \begin{tabular}{| c | c c c | c c c |}
         \hline
				  & \multicolumn{3}{c|}{HIGGS} & \multicolumn{3}{c|}{PARITY} \\
				  & Full & no $f^T$ & DAUB & Full & no $f^T$ & DAUB \\	\hline \hline
        Iterations		& 41 & 105 & 56 & 41 & 26 & 19 \\
        Allocation  & 1,578k  & 590k & {\bf 372k}   & 860k  & 171k & {\bf 156k}  \\
        Time (sec)	& 49,905 & 17,186 & {\bf 2,001} & 5,939 & 617 &{\bf 397} \\
        \hline
\end{tabular}
\end{table}



Results for HIGGS and PARITY are as follows.  The accuracy loss of the
ultimate classifiers selected by DAUB turned out to be quite small:
DAUB selected the top classifier for HIGGS (i.e. 0.0\% loss) and one
of the top three classifiers for PARITY (0.3\% loss).
In terms of complexity reduction,
Table~\ref{tab:performance} shows clear gains over ``full'' training
of all classifiers on the full $T_r$, in both total allocated
samples as well total CPU training time, for both standard DAUB as well
as a variant which does not use training set accuracy $f^T$ as an
upper bound.  Both variants reduce the allocated samples by $\sim$2x-4x for
HIGGS, and by $\sim$5x for PARITY.
The impact on CPU runtime is more pronounced,
as many sub-optimal classifiers with supra-linear runtimes
receive very small amounts of training data.
As the table shows, standard DAUB reduces total training time
by a factor of $\sim$25x for HIGGS, and $\sim$15x for PARITY.

\begin{figure}[htb]
\centering
  \includegraphics[width=0.55\textwidth]{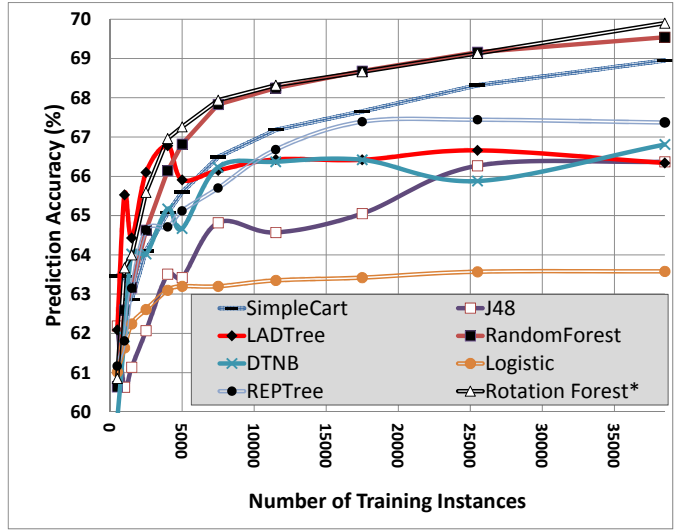}
  \caption{Higgs validation accuracy curves of several classifiers that initially outperform the classifier (Rotation Forest) that is ultimately the best when given all training data.}
  \label{FIG:HiggsPerformance}
\end{figure}

Figures~\ref{FIG:HiggsPerformance} and~\ref{FIG:DaubAllocationHiggs}
provide additional insight into DAUB's behavior.
Figure~\ref{FIG:HiggsPerformance} shows how validation accuracy progresses
with increasing training data allocation to several classifiers on the
HIGGS dataset.  The plots for the most part conform to our ideal-case
assumptions of increasing accuracy with diminishing slope,
barring a few monotonicity glitches\footnote{
In our experience, most of these glitches pertain to weak classifiers
and thus would not significantly affect DAUB, since DAUB mostly focuses
its effort on the strongest classifiers.}
due to stochastic sampling noise.  
We note that, while there is one optimal classifier $C^*$ (a parameterization
of a Rotation Forest) with best validation accuracy after training
on all of $T_r$, there are several other 
classifiers that outperformed $C^*$ in early training.
For instance, LADTree is better than $C^*$ until 
5,000 examples but then flattens out.

\begin{figure}[htb]
  \centering
  \includegraphics[width=0.6\textwidth]{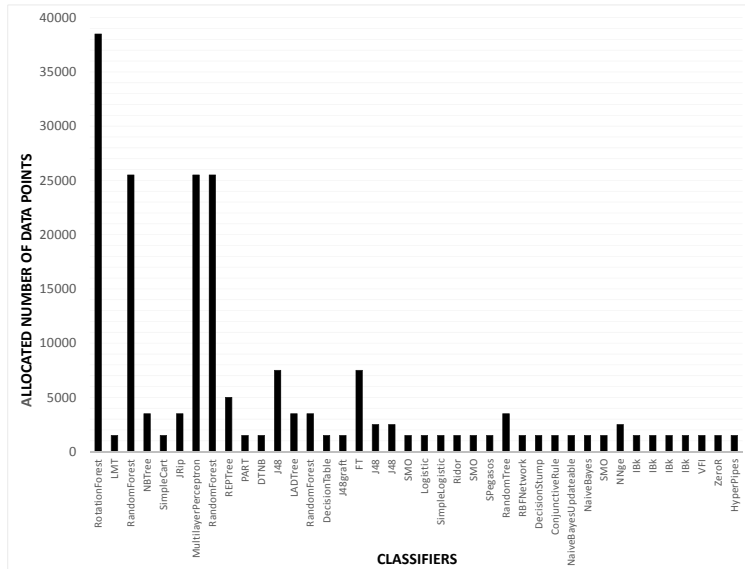}
  \caption{Number of Higgs examples allocated by DAUB to various
  classifiers, sorted by decreasing validation performance at $N=38500$.}
  \label{FIG:DaubAllocationHiggs}
\end{figure}

 Figure~\ref{FIG:DaubAllocationHiggs} gives perspective on how DAUB
distributes data allocations among the 41 classifiers when run on the
HIGGS dataset. The classifiers here are sorted by decreasing validation
accuracy $f_i(N)$.
While DAUB manages to select $C^*$ in this case, what's equally critical
is the distribution of allocated training data. The figure shows that
DAUB allocates most of the training data to the top eight classifiers.
Most classifiers receive 2500 or fewer samples, and only four
classifiers receive more than 10k samples, with all of them within
1.5\% of the optimal performance.

\begin{table*}[htb]
  \setlength{\doublerulesep}{\arrayrulewidth}
  \setlength{\tabcolsep}{0.9ex}
  \centering
  \begin{tabular}{| l l || c  r | c c r c c |}
    \hline
     & & \multicolumn{2}{c|}{Full Training} & \multicolumn{5}{c|}{DAUB} \\
    Dataset & Application Area & Alloc. & Time (s) & Iter. & Alloc. & Time (s) & Speedup & Loss \\
    \hline
    Buzz                  & social media & 1,578k & 56,519 & 57 & 302k & 5,872 & 10x & 0.0\% \\
    Cover Type        & forestry & 1,578k & 43,578 & 13 & 160k & 3,848 & 11x & 1.1\% \\
    HIGGS                & signal processing & 1,578k & 49,905 & 56 & 372k & 2,001 & 25x & 0.0\% \\
    Million Songs    & music & 1,578k & 115,911 & 53 & 333k & 17,208 & \ \ 7x & 0.6\% \\
    SUSY                  & high-energy physics & 1,578k & 26,438 & 31 & 214k & 837 & 31x & 0.9\% \\
    Vehicle Sensing & vehicle mgmt. & 1,578k & 68,139 & 50 & 296k & 5,603 & 12x & 0.0\% \\
    \hline
  \end{tabular}
  \caption{\label{tab:performancebuzz} Comparison of full training and DAUB
  across six benchmarks.}  
\end{table*}

Finally, in Table~\ref{tab:performancebuzz} we report results of DAUB
on Higgs plus five other real-world benchmarks as indicated: 
Buzz~\cite{buzz};
Covertype~\cite{covtype};
Million Song Dataset~\cite{Bertin-Mahieux2011}; 
SUSY~\cite{HiggsDataset2014}; and
Vehicle-SensIT~\cite{vehicle}.
These experiments use exactly the same parameter settings
as for HIGGS and PARITY.
As before, the table shows a comparison in terms of allocated training
samples and runtime. In addition it displays the incurred accuracy loss
of DAUB's final selected classifier. The highest loss is $\sim$1\%,
well within an acceptable range. The average incurred loss across all
six benchmarks is 0.4\% and the average speedup is 16x. Our
empirical findings thus show that in practice DAUB can consistently select
near-optimal classifiers at a substantial reduced computational cost when
compared to full training of all classifiers.

\section{Conclusion}

We reiterate the potential practical impact of
our original \emph{Cost-Sensitive Training Data Allocation} problem
formulation, and our proposed DAUB algorithm for solving
this problem.  In our experience, DAUB has been quite easy to use,
easy to code and tune, and is highly practical in robustly finding
near-optimal learners with greatly reduced CPU time across
datasets drawn from a variety of
real-world domains.  Moreover, it does not require built-in knowledge
of learners or properties of datasets, making it ideally suited for
practitioners without domain knowledge of the learning algorithms
or data characteristics. Furthermore,
all intermediate results can be used to interactively inform the 
practitioner of relevant information such as progress (e.g., updated
learning curves) and decisions taken (e.g., allocated data). Such a
tool was introduced by~\namecite{TowardsCADS-2015} and a
snapshot of it is depicted in Figure~\ref{FIG:DaubTool}.

\begin{figure}[htb]
  \centering
  \includegraphics[width=0.7\textwidth]{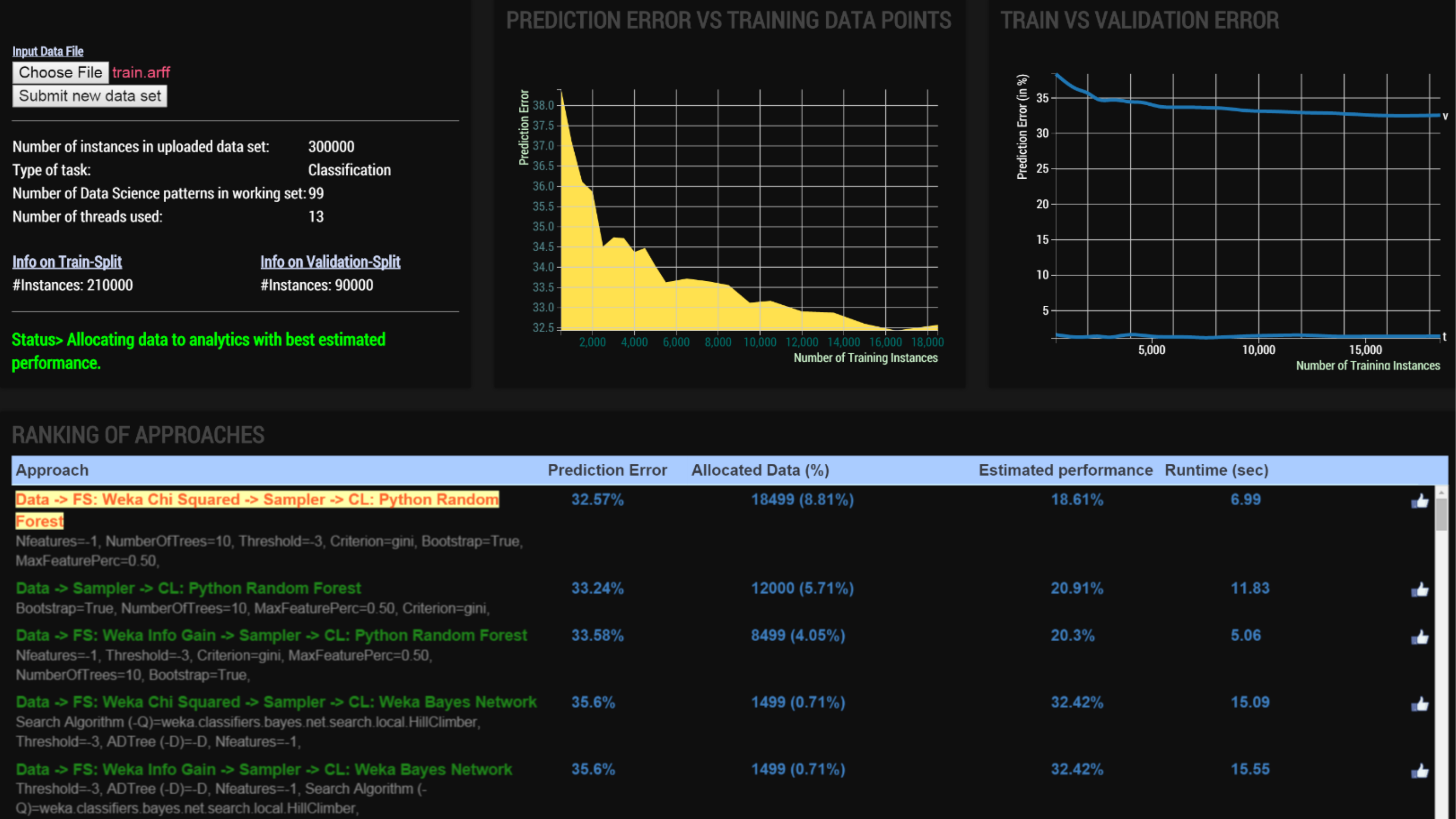}
  \caption{Implementation of a practitioners tool based on DAUB. The interactive dashboard shows the
   performance of various learners in terms of loss on the validation set based on training on currently allocated training data points. For each learner there are continuously updated plots that show the learning curve and a comparison of training and validation performance.}
  \label{FIG:DaubTool}
\end{figure}

 Our theoretical work on the idealized DAUB* scenario also reveals
novel insights and provides important support for the real-world
behavior of DAUB with noisy accuracy estimates.  As dataset sizes
scale, we expect that DAUB will better and better approach the
idealized behavior of DAUB*, which offers strong bounds on
both learner sub-optimality as well as regret due to misallocated
samples.

There are many opportunities for further advances in both the
theoretical and practical aspects of this work.
It should be possible
to develop more accurate bound estimators given noisy accuracy estimates,
e.g., using monotone spline regression.
Likewise, it may be possible to extend the theory to
encompass noisy accuracy estimates, for example, by making use of
PAC lower bounds on generalization error to establish upper bounds
on learner accuracy.
DAUB could be further combined in an interesting way with
methods~\cite{ActiveLEarningWithModelSelection}
to optimally split data between training and validation sets.




\begin{small}

\begin{thebibliography}{30}
\expandafter\ifx\csname natexlab\endcsname\relax\def\natexlab#1{#1}\fi
\expandafter\ifx\csname url\endcsname\relax
  \def\url#1{{\tt #1}}\fi

\bibitem[Agrawal and Goyal(2012)]{agrawal2011analysis}
S.~Agrawal and N.~Goyal.
\newblock Analysis of {T}hompson sampling for the multi-armed bandit problem.
\newblock In {\em COLT-2012}, pp. 39.1--39.26, Edinburgh, Scotland, June 2012.

\bibitem[Ali et~al.(2014)Ali, Caruana, and
  Kapoor]{ActiveLEarningWithModelSelection}
A.~Ali, R.~Caruana, and A.~Kapoor.
\newblock Active learning with model selection.
\newblock In {\em Proc.~of AAAI-2014}, 2014.

\bibitem[Auer et~al.(2002)Auer, Cesa-Bianchi, and Fischer]{acf02:ucb1J}
P.~Auer, N.~Cesa-Bianchi, and P.~Fischer.
\newblock Finite-time analysis of the multiarmed bandit problem.
\newblock {\em Machine Learning}, 47\penalty0 (2-3):\penalty0 235--256, 2002.

\bibitem[Baldi et~al.(2014)Baldi, Sadowski, and Whiteson]{HiggsDataset2014}
P.~Baldi, P.~Sadowski, and D.~Whiteson.
\newblock Searching for exotic particles in high-energy physics with deep
  learning.
\newblock {\em Nature Communications}, 5, July 2014.

\bibitem[Bergstra et~al.(2011)Bergstra, Bardenet, Bengio, and
  K{\'e}gl]{algsforhyperoptimization}
J.~Bergstra, R.~Bardenet, Y.~Bengio, and B.~K{\'e}gl.
\newblock Algorithms for hyper-parameter optimization.
\newblock In {\em NIPS}, pp. 2546--2554, 2011.

\bibitem[Bergstra and Bengio(2012)]{randomsearch}
J.~Bergstra and Y.~Bengio.
\newblock Random search for hyper-parameter optimization.
\newblock {\em Journal of Machine Learning Research}, 13:\penalty0 281--305,
  2012.

\bibitem[Bergstra et~al.(2013)Bergstra, Yamins, and Cox]{MSMS}
J.~Bergstra, D.~Yamins, and D.~D. Cox.
\newblock Making a science of model search: Hyperparameter optimization in
  hundreds of dimensions for vision architectures.
\newblock In {\em ICML-2013}, 2013.

\bibitem[Bertin-Mahieux et~al.(2011)Bertin-Mahieux, Ellis, Whitman, and
  Lamere]{Bertin-Mahieux2011}
T.~Bertin-Mahieux, D.~P. Ellis, B.~Whitman, and P.~Lamere.
\newblock The million song dataset.
\newblock In {\em {Proceedings of the 12th International Conference on Music
  Information Retrieval ({ISMIR} 2011)}}, 2011.

\bibitem[Biem et~al.(2015)Biem, Butrico, Feblowitz, Klinger, Malitsky, Ng,
  Perer, Reddy, Riabov, Samulowitz, Sow, Tesauro, and Turaga]{TowardsCADS-2015}
A.~Biem, M.~A. Butrico, M.~D. Feblowitz, T.~Klinger, Y.~Malitsky, K.~Ng,
  A.~Perer, C.~Reddy, A.~V. Riabov, H.~Samulowitz, D.~Sow, G.~Tesauro, and
  D.~Turaga.
\newblock Towards cognitive automation of data science.
\newblock In {\em Proc.~of AAAI-2015, Demonstrations Track}, Austin, TX, 2015.

\bibitem[Blackard and Dean(2000)]{covtype}
J.~A. Blackard and D.~J. Dean.
\newblock Comparative accuracies of artificial neural networks and discriminant
  analysis in predicting forest cover types from cartographic variables.
\newblock {\em Computers and Electronics in Agriculture}, 24\penalty0
  (3):\penalty0 131--151, 2000.

\bibitem[Brochu et~al.(2009)Brochu, Cora, and de~Freitas]{bayesianoptimization}
E.~Brochu, V.~M. Cora, and N.~de~Freitas.
\newblock A tutorial on {B}ayesian optimization of expensive cost functions,
  with application to active user modeling and hierarchical reinforcement
  learning.
\newblock Technical Report UBC TR-2009-23, Department of Computer Science,
  University of British Columbia, 2009.

\bibitem[Caruana and Niculescu-Mizil(2006)]{caruana2006empirical}
R.~Caruana and A.~Niculescu-Mizil.
\newblock An empirical comparison of supervised learning algorithms.
\newblock In {\em ICML-2006}, pp. 161--168, 2006.

\bibitem[Domhan et~al.(2015)Domhan, Springenberg, and Hutter]{domhanspeeding}
T.~Domhan, J.~T. Springenberg, and F.~Hutter.
\newblock Speeding up automatic hyperparameter optimization of deep neural
  networks by extrapolation of learning curves.
\newblock In {\em IJCAI-2015}, 2015.

\bibitem[Duarte and Hu(2004)]{vehicle}
M.~Duarte and Y.~H. Hu.
\newblock Vehicle classification in distributed sensor networks.
\newblock In {\em {Journal of Parallel and Distributed Computing}}, 2004.

\bibitem[Feurer et~al.(2015)Feurer, Springenber, and
  Hutter]{HyperOptviaMetaLearning}
M.~Feurer, J.~Springenber, and F.~Hutter.
\newblock Initializing bayesian hyperparameter optimization via meta-learning.
\newblock In {\em Proc.~of AAAI-2015}, 2015.

\bibitem[Guerra et~al.(2008)Guerra, Prudencio, and
  Ludermir]{PredictingModel2008}
S.~B. Guerra, R.~B.~C. Prudencio, and T.~B. Ludermir.
\newblock Predicting the performance of learning algorithms using support
  vector machines as meta-regressors.
\newblock In {\em ICANN}, 2008.

\bibitem[Guyon et~al.(2006)Guyon, Alamdari, Dror, and
  Buhmann]{Guyon2006:challenge}
I.~Guyon, A.~R. S.~A. Alamdari, G.~Dror, and J.~M. Buhmann.
\newblock Performance prediction challenge.
\newblock In {\em IJCNN-2006}, pp. 1649--1656, Vancouver, BC, Canada, July
  2006.

\bibitem[Hall et~al.(2009)Hall, Frank, Holmes, Pfahringer, Reutemann, and
  Witten]{weka}
M.~Hall, E.~Frank, G.~Holmes, B.~Pfahringer, P.~Reutemann, and I.~H. Witten.
\newblock The {WEKA} {D}ata {M}ining {S}oftware: An update.
\newblock {\em {SIGKDD} {E}xplorations}, 11\penalty0 (1), 2009.

\bibitem[Hoffman et~al.(2014)Hoffman, Shahriari, and
  de~Freitas]{BayesianBandits}
M.~D. Hoffman, B.~Shahriari, and N.~de~Freitas.
\newblock On correlation and budget constraints in model-based bandit
  optimization with application to automatic machine learning.
\newblock In {\em AISTATS}, pp. 365--374, 2014.

\bibitem[Hutter et~al.(2014)Hutter, Xu, Hoos, and
  Leyton-Brown]{algorithmruntimeprediction}
F.~Hutter, L.~Xu, H.~H. Hoos, and K.~Leyton-Brown.
\newblock Algorithm runtime prediction: Methods \& evaluation.
\newblock {\em Artif. Intell.}, 206:\penalty0 79--111, 2014.

\bibitem[Kawala et~al.(2013)Kawala, Douzal-Chouakria, Gaussier, and
  Dimert]{buzz}
F.~Kawala, A.~Douzal-Chouakria, E.~Gaussier, and E.~Dimert.
\newblock Pr\'{e}dictions d'activit\'{e} dans les r\'{e}seaux sociaux en ligne.
\newblock In {\em {Conf\'{e}rence sur les Mod\'{e}les et l′Analyse des
  R\'{e}seaux Approches Math\'{e}matiques et Informatique (MARAMI)}}, 2013.

\bibitem[McCallum(2002)]{mallet}
A.~McCallum.
\newblock {MALLET}: A machine learning for language toolkit.
\newblock http://mallet.cs.umass.edu, 2002.

\bibitem[Munos(2014)]{munos2013:optimism}
R.~Munos.
\newblock From bandits to {Monte-Carlo Tree Search}: The optimistic principle
  applied to optimization and planning.
\newblock {\em Foundations and Trends in Machine Learning}, 7(1):\penalty0
  1--130, 2014.

\bibitem[Pedregosa et~al.(2011)]{scikit-learn}
F.~Pedregosa et~al.
\newblock Scikit-learn: Machine learning in {P}ython.
\newblock {\em JMLR}, 12:\penalty0 2825--2830, 2011.

\bibitem[Rice(1976)]{Rice}
J.~Rice.
\newblock The algorithm selection problem.
\newblock {\em Advances in Computers}, 15:\penalty0 65--118, 1976.

\bibitem[Schaul et~al.(2010)Schaul, Bayer, Wierstra, Sun, Felder, Sehnke,
  R\"{u}ckstie\ss, and Schmidhuber]{pybrain}
T.~Schaul, J.~Bayer, D.~Wierstra, Y.~Sun, M.~Felder, F.~Sehnke,
  T.~R\"{u}ckstie\ss, and J.~Schmidhuber.
\newblock {PyBrain}.
\newblock {\em JMLR}, 2010.

\bibitem[Scott(2010)]{scott2010bayesbandits}
S.~L. Scott.
\newblock A modern {B}ayesian look at the multi-armed bandit.
\newblock {\em Appl. Stochastic Models Bus. Ind.}, 26:\penalty0 639--658, 2010.

\bibitem[Snoek et~al.(2012)Snoek, Larochelle, and Adams]{PAOML}
J.~Snoek, H.~Larochelle, and R.~P. Adams.
\newblock Practical {B}ayesian optimization of machine learning algorithms.
\newblock In {\em Advances in Neural Information Processing Systems 25}, 2012.

\bibitem[Thompson(1933)]{thompson33:sampling}
W.~R. Thompson.
\newblock On the likelihood that one unknown probability exceeds another in
  view of the evidence of two samples.
\newblock {\em Biometrika}, 25:\penalty0 285--294, 1933.

\bibitem[Thornton et~al.(2013)Thornton, Hutter, Hoos, and
  Leyton-Brown]{AutoWEKA}
C.~Thornton, F.~Hutter, H.~H. Hoos, and K.~Leyton-Brown.
\newblock Auto-{WEKA}: Combined selection and hyperparameter optimization of
  classification algorithms.
\newblock In {\em Proc.~of KDD-2013}, pp. 847--855, 2013.

\end{thebibliography}

\end{small}


\appendix

\section{Appendix: Proof Details}

\begin{lemma}
  \label{lem:data-allocation-new}
  Let $\mathcal{C}, T_r, T_v, N, M,$ and $f_i$ for $i \in [M]$ be as
  in Definition~\ref{def:problem}. Let $r > 1$ and $b \in
  \mathbb{N}^+$.  If the projected upper bound functions $u_j$ and
  $u_{i*}$ used by DAUB*$(\mathcal{C}, T_r, T_v, r, b)$ are valid, then
  it allocates to $C_j$ fewer than $r n^*_j$ examples in each step and
  $\frac{r^2}{r-1} n^*_j$ examples in total.
\end{lemma}

\begin{proof}[Proof of Lemma~\ref{lem:data-allocation-new}]
  Suppose, for the sake of contradiction, that DAUB* allocates at least
  $r n^*_j$ examples to learner $C_j$ at some point in its
  execution. Since $r > 1$ and all allocation sizes are at most $N$,
  $n^*_j < N$. Further, since $n_j$, the number of
  examples allocated to $C_j$, is always incremented geometrically by
  a factor of at most $r$, at some previous point in the algorithm, we
  must have $n_j \in \{n^*_j, \ldots, r n^*_j - 1\}$. Since the
  projected upper bound function $u_j$ is non-increasing, at that
  previous point in the algorithm, $u_j(n_j) \leq u_j(n^*_j) < f^*$ by
  the definition of $n^*_j$. On the other hand, since the projected
  upper bound function $u_{i^*}$ is valid, the projected upper bound
  for $C_{i^*}$ would always be at least $f^*$.

  Therefore, the algorithm, when choosing which
  learner to allocate the next set of examples to, will, from this point
  onward, always prefer $C_{i^*}$ (and possibly another
  learner appearing to be even better) over $C_j$, implying that
  $n_j$ will never exceed its current value. This contradicts the
  assumption that DAUB* allocates at least $r n^*_j$ examples to $C_j$
  at some point during its execution.

  For the bound on the total number of examples allocated to $C_j$,
  let $k = \lfloor \log_r \frac{r n^*_j}{b} \rfloor$. Since DAUB*
  allocates fewer than $r n^*_j$ examples to $C_j$ in any single step
  and the allocation sizes start at $b$ and increase by a factor of
  $r$ in each step, $C_j$ must have received precisely $b + b r + b
  r^2 + \ldots + b r^k$ examples in total. This is smaller than
  $\frac{r}{r-1} b r^k \leq \frac{r^2}{r-1} n^*_j$.
\end{proof}

\begin{lemma}
  \label{lem:data-vs-cost-new}
  Let $\mathcal{C}, T_r, T_v, N, M,$ and $c_i$ for $i \in [M]$ be as
  in Definition~\ref{def:problem}. Let $r > 1, b \in \mathbb{N}^+,$
  and $S$ be the training data allocation sequence produced by
  DAUB*$(\mathcal{C}, T_r, T_v, r, b)$. If DAUB* allocates at most $n$
  training examples to a learner $C_j \in \mathcal{C}$ in each
  step, then $\cost(S_j) \leq \frac{r}{r-1} c_j(n)$.
\end{lemma}

\begin{proof}[Proof of Lemma~\ref{lem:data-vs-cost-new}]
  As in the proof of Lemma~\ref{lem:data-allocation-new}, let $k = \lfloor
  \log_r (n/b) \rfloor$ and observe that the data allocation
  subsequence $S_j$ for learner $C_j$ must have been $(b, b r, b
  r^2, \ldots, b r^k)$. The corresponding training cost for $C_j$ is
  $\cost(S_j) = c_j(b) + c_j(b r) + c_j(b r^2) + \ldots c_j(b
  r^k)$. By the assumption that $c_j$ grows at least linearly:
$
    c_j(b r^\ell) \leq \frac{c_j(b r^k)}{r^{k-\ell}}
$
  for $\ell \leq k$. It follows that:
 \begin{align*}
    \cost(S_j)
    & \leq c_j(b r^k) \cdot (r^{-k} + r^{-k+1} + r^{-k+2} + \ldots + 1) \\
    & < \frac{r}{r-1} c_j(b r^k) 
        \leq \frac{r}{r-1} c_j(n)
\end{align*}
This finishes the proof.
\end{proof}

\begin{lemma}
  \label{lem:ub}
  If $f_i(n)$ is well-behaved, then $\ub_i(n)$ is a valid projected
  upper bound function.
\end{lemma}

\begin{proof}[Proof of Lemma~\ref{lem:ub}]
  Recall that $\ub_i(n)$ is the minimum of $f^T_i(n)$ and $f_i(n) +
  (N-n) f'_i(n)$.  Since $f^T_i(n)$ is a non-increasing function and
  is already argued to be an upper bound on $f_i(N)$, it suffices to
  show that $g(n) = f_i(n) + (N-n) f'_i(n)$ is also a non-increasing
  function of $n$ and $g(n) \geq f_i(N)$.

  \begin{align*}
    g(n+1)
    & = f_i(n+1) + (N-n-1) f'_i(n+1) \\
    & \leq \left( f_i(n) + f'_i(n) \right) + (N-n-1) f'_i(n+1) \\
    & \leq \left( f_i(n) + f'_i(n) \right) + (N-n-1) f'_i(n)
      && \text{\indent because $f'_i$ is non-increasing}\\
    & = f_i(n) + (N-n) f'_i(n)\\
    & = g(n)
  \end{align*}
  The first inequality follows from the assumptions on the
  behavior of $f_i$ w.r.t.\ $n$ and its first-order Taylor
  expansion. Specifically, recall that $f'_i(n)$ is defined as
  $(f_i(n) - f_i(n-s))/s$ for some (implicit) parameter $s$. For
  concreteness, let's refer to that implicit function as
  $h_i(n,s)$. Let $h'_i(n,s)$ denote the discrete derivative of
  $h_i(n,s)$ w.r.t.\ $n$. The non-increasing nature of $f'_i(n)$
  w.r.t.\ $n$ implies $h'_i(n,s)$, for any fixed $n$, is a
  non-decreasing function of $s$. In particular, $f'_i(n) = h'_i(n,s)
  \geq h'_i(n,1)$ for any $s \geq 1$.  It follows that $f_i(n+1) =
  f_i(n) + h'_i(n+1,1) \leq f_i(n) + h'_i(n+1,s) = f_i(n) + f'_i(n+1)
  \leq f_i(n) + f'_i(n)$, as desired.

  Thus, $g(n)$ is non-increasing. Since $g(N) = f(N)$ by
  definition, we have $g(n) \geq f(N)$, finishing the proof.
\end{proof}

\begin{lemma}
  \label{lem:n-delta}
  For an $(N,\Delta)$-suboptimal learner $C_j$ with well-behaved
  accuracy function $f_j$, $n^*_j \leq n^\Delta_j$.
\end{lemma}

\begin{proof}[Proof of Lemma~\ref{lem:n-delta}]
  If $f'_j(N) > \Delta/N$, then $n^\Delta_j = N$ and the statement of
  the lemma holds trivially.  Otherwise, by the definition of
  $n^\Delta_j$, $f'_j(n^\Delta_j) \leq \Delta/N$.  In order to prove
  $n^*_j \leq n^\Delta_j$, we must show that $\ub_j(n^\Delta_j) <
  f^*$.  We do this by using first-order Taylor expansion:
  \begin{align*}
  \ub_j(n^\Delta_j)
  & \leq f_j(n^\Delta_j) + (N - n^\Delta_j) \, f'_j(n^\Delta_j) \\
  & \leq f_j(n^\Delta_j) + (N - n^\Delta_j) \, \frac{\Delta}{N}
    && \text{by the above observation} \\
  & < f_j(n^\Delta_j) + \Delta\\
  & \leq f_j(N) + \Delta
    && \text{since $n^\Delta_j \leq N$ and $f_j$ is non-decreasing} \\
  & \leq f_{i^*}(N)
    && \text{by $(N,\Delta)$-suboptimality of $C_j$}
  \end{align*}
  Hence, $\ub_j(n^\Delta_j) < f_{i^*}(N) = f^*$.
\end{proof}

\begin{proposition}
  \label{prop:fast-decay}
  If $g : \mathbb{N} \to [m_1,m_2]$ is a well-behaved function for
  $m_1,m_2 \in \mathbb{R}$, then its discrete derivative $g'(n)$
  decreases asymptotically as $\ltlOh(1/n)$.
\end{proposition}

\begin{proof}[Proof of Proposition~\ref{prop:fast-decay}]
  Recall the definition of the discrete derivative from
  Section~\ref{sec:optimistic} and assume for simplicity of exposition
  that the parameter $s$ is $1$, namely, $g'(n) = g(n) - g(n-1)$.  The
  argument can be easily extended to $s > 1$. Applying the definition
  repeatedly, we get $g(n) = g(1) + (g'(2) + \ldots + g'(n)) \geq m_1
  + g'(2) + \ldots g'(n)$.  If $g'(n)$ was $\Omega(1/n)$, then there
  would exist $n_0$ and $c$ such that for all $n \geq n_0$, $g'(n)
  \geq c/n$. This would mean $g(n) \geq m_1 + \sum_{n \geq n_0}
  c/n$. This summation, however, diverges to infinity while $g(n)$ is
  bounded above by $m_2$. It follows that $g'(n)$ could not have been
  $\Omega(1/n)$ to start with.  It must thus decrease asymptotically
  strictly faster than $1/n$, that is, be $\ltlOh(1/n)$.
\end{proof}

\begin{lemma}
  \label{lem:n-delta-asymptotic}
  For an $(N,\Delta)$-suboptimal learner $C_j$ with a well-behaved
  accuracy function $f_j$ satisfying $f'_j(N) \leq \Delta/N$, we have that
  $n^\Delta_j$ is $\ltlOh(N)$ in $N$.
\end{lemma}

\begin{proof}[Proof of Lemma~\ref{lem:n-delta-asymptotic}]
  From Proposition~\ref{prop:fast-decay}, $f'_j(N) = \ltlOh(1/N)$,
  implying $\Delta/f'_j(N) = \omega(N \Delta)$. This means that a
  value $N'$ that is $\ltlOh(N/\Delta)$ and suffices to ensure
  $\Delta/f'_j(N') \geq N$ for all large enough $N$, that is, $f'_j(N')
  \leq \Delta/N$. Since $n^\Delta_j$ is, by definition, no larger than
  $N'$, $n^\Delta_j$ must also be $\ltlOh(N/\Delta)$.
\end{proof}

\begin{proof}[Proof of Theorem~\ref{thm:regret-bound-new}]
  Since DAUB* never allocates more than $N$ training examples in a
  single step to $C_i$ for any $i \in [M]$, it follows from
  Lemma~\ref{lem:data-vs-cost-new} that $\cost(S_i) \leq \frac{r}{r-1}
  c_i(N)$. In particular, $\cost(S_{\tilde{i}}) \leq \frac{r}{r-1}
  c_{\tilde{i}}(N)$.

  The $(N,\Delta)$-regret of DAUB*, by definition, is $\sum_{j \in J}
  \cost(S_j)$. By Theorem~\ref{thm:regret-actual}, this is at most
  $\frac{r}{r-1} \sum_{j \in J} c_j(r n^\Delta_j)$. Since the cost
  function $c_j$ is assumed to increase at least linearly, this
  quantity is at most $\frac{r}{r-1} \sum_{j \in J} \frac{r
    n^\Delta_j}{N} c_j(N)$. From Lemma~\ref{lem:n-delta-asymptotic},
  we have that $n^\Delta_j = \ltlOh(N/\Delta)$ and hence $\frac{r
    n^\Delta_j}{N} = \ltlOh(1)$ in $N$.  Plugging this in and dropping
  the constants $r$ and $\Delta$ from the asymptotics, we obtain that
  the regret is $\ltlOh(\sum_{j \in J} c_j(N))$.

  Finally, if $c_j(n) = \bigOh(c_{\tilde{i}}(n))$, then $\sum_{j \in
    J} c_j(N) = \bigOh(\sum_{j \in J} c_{\tilde{i}}(N))$, which is
  simply $\bigOh(M \cdot c_{\tilde{i}}(N))$. Since
  $\cost(S_{\tilde{i}}) \geq c_{\tilde{i}}(N)$, this quantity is also
  $\bigOh(M \cdot \cost(S_{\tilde{i}}))$ in $N$. It follows from the
  above result that the $(N,\Delta)$-regret of DAUB* is $\ltlOh(M \cdot
  \cost(S_{\tilde{i}}))$ in $N$, as claimed.
\end{proof}

\begin{theorem}[Lower Bound, formal statement]
  \label{thm:data-allocation-LB}
  Let $\Delta \in (0,1]$ and $\mathcal{A}$ be a training data
  allocation algorithm that, when executed on a training set of size
  $N$, is guaranteed to always output an $(N,\Delta)$-optimal
  learner. Let $\mathcal{C}, T_r, T_v, N, M, c_i$ and $f_i$ for $i
  \in [M]$ be as in Definition~\ref{def:problem}. Let $\gamma =
  (\sqrt{5}-1)/2$ and $C_j \in \mathcal{C}$ be an
  $(N,\Delta)$-suboptimal learner. Then there exists a choice of
  $f_j(n)$ such that $f_j$ is well-behaved, $f'_j(N) \leq \Delta/N,$
  and $\mathcal{A}(\mathcal{C}, T_r, T_v)$ allocates to $C_j$ more
  than $\gamma n^\Delta_j$ examples, thus incurring a misallocated
  training cost on $C_j$ larger than $\gamma c_j(n^\Delta_j)$.
\end{theorem}

\begin{proof}[Proof of Theorem~\ref{thm:data-allocation-LB}]
  We will argue that, under certain circumstances,
  $\mathcal{A}$ must allocate at least $\gamma N^\Delta_j$ examples
  to $C_j$ in order to guarantee $(N,\Delta)$ optimality.

  To prove the desired result by
  contradiction, we consider the optimal learner $C_{i^*}$ and
  will construct specific accuracy functions $f_j(n)$ and $f_{i^*}(n)$
  such that they are identical for all $n \leq \gamma N^\Delta_j$, have
  derivative no more than $\Delta/n$ at $N^\Delta_j$, but differ by
  at least $\Delta$ when evaluated at $n=N$. This would imply that
  $\mathcal{A}$ simply cannot distinguish between the accuracy of
  $C_j$ and $C_{i^*}$ by evaluating them on at most $\gamma N^\Delta_j$
  examples and thus cannot guarantee $(N,\Delta)$-optimality of the
  learner it outputs.

  Suppose we can preserve the properties of $f_j$ required by the
  theorem, including $f'_j(N^\Delta_j) \leq \Delta/N$, but can enforce that
  $f'_j(\gamma N^\Delta_j) \geq d \Delta/N$ for some $d > 1$ whose
  value we will determine shortly. Further, let us alter $f_j$ such
  that it remains unchanged for $n \leq \gamma N^\Delta_j$ and
  is altered for larger $n$ such that
  $f'_j(n) = d \Delta/N$ for $\gamma N^\Delta_j < n \leq N^\Delta_j$
  and $f'_j(n) = \Delta/N$ for $n > N^\Delta_j$. Recalling that
  $N^\Delta_j \leq N$, this modified $f_j$ then satisfies:
  \begin{align*}
    f_j(N) - f_j(\gamma N^\Delta_j)
    & = (N^\Delta_j - \gamma N^\Delta_j) \frac{d \Delta}{N} + (N-N^\Delta_j) \frac{\Delta}{N}\\
    & = (1 - \gamma - 1/d) N^\Delta_j \frac{d \Delta}{N} + \Delta
  \end{align*}
  which is at least $\Delta$ as long as $1/d \leq 1 - \gamma$.
  Now define $f_{i^*}$ s.t.\ $f_{i^*}(n) = f_j(n)$ for $n \leq
  \gamma N^\Delta_j$ and $f_{i^*}(n) = f_j(\gamma N^\Delta_j)$ for $n >
  \gamma N^\Delta_j$. Thus, $f_j$ and $f_{i^*}$ are
  identical for $n \leq \gamma N^\Delta_j$ but $f_{i^*}(N) - f_j(N)
  \geq \Delta$, as desired.

  It remains to show that we can choose a valid $f_j$ satisfying the
  properties required by the theorem but such that $f'_j(\gamma N^\Delta_j)
  \geq d \Delta/N$. To achieve this, consider $f_j(n) = 1 - c
  \Delta/n$ for any $c>0$. Then $f_j$ is non-decreasing, $f'_j(n) = c
  \Delta / n^2$ is non-increasing, $n^\Delta_j = \sqrt{cN}$,
  $f'_j(n^\Delta_j) = \Delta/N$, and $f'_j(\gamma n^\Delta_j) = f'_j(\gamma
  \sqrt{cN}) = \Delta / (\gamma^2 N)$. Notice that $\gamma = (\sqrt{5}
  - 1) / 2$ satisfies the condition $\gamma^2 \leq 1-\gamma$;
  it is in fact the largest value that satisfies the condition.
  Hence, we can choose $d = 1/\gamma^2$ in order to ensure
  $1/d \leq 1 - \gamma$, which in turn ensures that
  the altered $f_j(N)$ and $f_j(\gamma N)$, as constructed above,
  differ by at least $\Delta$. This finishes the construction for the
  lower bound.
\end{proof}

\section{Classifier Parameterizations}

The exact configurations of each classifier in the WEKA package that
were used in the experiments are as follows:

\begin{enumerate}
\setlength\itemsep{-2pt}
\item weka.classifiers.trees.SimpleCart 
\item weka.classifiers.rules.DecisionTable 
\item weka.classifiers.bayes.NaiveBayesUpdateable 
\item weka.classifiers.functions.SMO -K ``weka.classifiers.functions.supportVector.RBFKernel -C 250007 -G 0.01''
\item weka.classifiers.functions.SMO -K ``weka.classifiers.functions.supportVector.PolyKernel -C 250007 -E 1.0''
\item weka.classifiers.functions.SMO -K ``weka.classifiers.functions.supportVector.NormalizedPolyKernel -C 250007 -E 2.0''
\item weka.classifiers.functions.RBFNetwork 
\item weka.classifiers.trees.RandomTree 
\item weka.classifiers.trees.RandomForest -depth 10 -I 5 -K 0 
\item weka.classifiers.trees.RandomForest -depth 10 -I 10 -K 0 
\item weka.classifiers.trees.RandomForest -depth 20 -I 5 -K 0 
\item weka.classifiers.rules.DTNB 
\item weka.classifiers.trees.NBTree 
\item weka.classifiers.functions.Logistic 
\item weka.classifiers.functions.SPegasos 
\item weka.classifiers.rules.PART 
\item weka.classifiers.trees.LADTree 
\item weka.classifiers.misc.HyperPipes 
\item weka.classifiers.trees.J48graft 
\item weka.classifiers.rules.JRip 
\item weka.classifiers.trees.BFTree 
\item weka.classifiers.trees.LMT 
\item weka.classifiers.trees.J48 -M 2 -C 0.25 
\item weka.classifiers.trees.J48 -M 4 -C 0.10 
\item weka.classifiers.trees.J48 -M 4 -C 0.35 
\item weka.classifiers.lazy.IBk -K 1 
\item weka.classifiers.lazy.IBk -K 5 
\item weka.classifiers.lazy.IBk -K 10 
\item weka.classifiers.lazy.IBk -K 25 
\item weka.classifiers.meta.RotationForest 
\item weka.classifiers.rules.ConjunctiveRule 
\item weka.classifiers.trees.REPTree 
\item weka.classifiers.rules.NNge 
\item weka.classifiers.rules.ZeroR 
\item weka.classifiers.trees.DecisionStump 
\item weka.classifiers.rules.Ridor 
\item weka.classifiers.misc.VFI 
\item weka.classifiers.bayes.NaiveBayes 
\item weka.classifiers.functions.MultilayerPerceptron 
\item weka.classifiers.functions.SimpleLogistic  
\item weka.classifiers.trees.SimpleCart 
\end{enumerate}

\end{document}